\documentclass[nohyperref]{article}

\usepackage{microtype}
\usepackage{graphicx}
\usepackage{subfigure}
\usepackage{booktabs} %
\usepackage{multirow}
\usepackage{adjustbox}

\usepackage{hyperref}

\usepackage[accepted]{icml2022}

\usepackage{wrapfig}
\usepackage{amsmath}
\usepackage{amssymb}
\usepackage{mathtools}
\usepackage{amsthm}
\usepackage{enumitem}

\usepackage{thmtools}
\usepackage{thm-restate}

\usepackage[capitalize,noabbrev]{cleveref}

\newcommand{\dist}{\Delta}

\renewcommand{\bar}{\overline}

\usepackage{amsmath,amsthm,amsfonts,bm}

\newtheorem{theorem}{Theorem}[section]
\newtheorem{proposition}{Proposition}[section]

\theoremstyle{definition}
\newtheorem*{proposition*}{Proposition}

\newtheorem{definition}{Definition}[section]

\def\eqref#1{equation~\ref{#1}}

\def\1{\bm{1}}

\def\vb{{\bm{b}}}

\def\ve{{\bm{e}}}

\DeclareMathAlphabet{\mathsfit}{\encodingdefault}{\sfdefault}{m}{sl}
\SetMathAlphabet{\mathsfit}{bold}{\encodingdefault}{\sfdefault}{bx}{n}

\def\gA{{\mathcal{A}}}

\def\gD{{\mathcal{D}}}

\def\gL{{\mathcal{L}}}
\def\gM{{\mathcal{M}}}

\def\gO{{\mathcal{O}}}

\def\gS{{\mathcal{S}}}

\newcommand{\E}{\mathbb{E}}

\newcommand{\R}{\mathbb{R}}

\DeclareMathOperator*{\argmax}{arg\,max}

\newcommand{\epomdp}{P(\gM | \gD)}
\newcommand{\algoname}{APE-V}

\icmltitlerunning{Offline RL Policies Should be Trained to be Adaptive}
\newcommand{\hQ}{\hat{Q}}

\usepackage{newfloat}

\DeclareFloatingEnvironment[fileext=lop]{diagram}
\begin{document}

\twocolumn[
\icmltitle{Offline RL Policies Should be Trained to be Adaptive}

\begin{icmlauthorlist}
\icmlauthor{Dibya Ghosh}{xxx}
\icmlauthor{Anurag Ajay}{yyy}
\icmlauthor{Pulkit Agrawal}{yyy}
\icmlauthor{Sergey Levine}{xxx}
\end{icmlauthorlist}
\icmlaffiliation{xxx}{UC Berkeley}
\icmlaffiliation{yyy}{MIT}

\icmlcorrespondingauthor{Dibya Ghosh}{dibya@berkeley.edu}

\icmlkeywords{Machine Learning, ICML}

\vskip 0.3in
]

\printAffiliationsAndNotice{} %
\begin{abstract}
Offline RL algorithms must account for the fact that the dataset they are provided may leave many facets of the environment unknown.
The most common way to approach this challenge is to employ pessimistic or conservative methods, which avoid behaviors that are too dissimilar from those in the training dataset. However, relying exclusively on conservatism has drawbacks: performance is sensitive to the exact degree of conservatism, and conservative objectives can recover highly suboptimal policies.
In this work, we propose that offline RL methods should instead be \textit{adaptive} in the presence of uncertainty. We show that acting optimally in offline RL in a Bayesian sense involves solving an implicit POMDP. As a result, optimal policies for offline RL must be adaptive, depending not just on the current state but rather all the transitions seen so far during evaluation.
We present a model-free algorithm for approximating this optimal adaptive policy, and demonstrate the efficacy of learning such adaptive policies in  offline RL benchmarks.
\end{abstract}

\section{Introduction}
\label{sec:intro}

Uncertainty about the environment is paramount in offline reinforcement learning (RL), the task of learning control behaviors from a dataset of logged environment interactions. Unlike in traditional RL, where the agent learns about all aspects of the environment through online interaction during training, offline RL methods must rely only on a fixed dataset, which often precludes the agent from identifying the exact dynamics of the environment. 

The most common approach for handling the resulting uncertainty is to learn conservative state-based policies incentivized to stay close to dataset behaviors. 
While these methods have been empirically successful with careful hyperparameter choices, it remains unclear 
whether conservative objectives are, in general, the best approach for designing offline RL algorithms. To formally understand this, we study the problem of offline RL under a Bayesian perspective. We show that when learning from an offline dataset of experience that does not fully specify the environment, uncertainty manifests as an \textit{implicit partial observability} in the original fully-observable learning problem. This implicit partial observability leads to a surprising consequence: a static state-based policy in offline RL, no matter how conservative, can be arbitrarily sub-optimal for maximizing test-time return, and in general, offline RL agents must \textit{adapt} to changes in the agent's uncertainty to act optimally. 

To understand where adaptation can help, consider what happens if an offline RL policy makes a mistake at evaluation time, e.g. taking an action that the agent incorrectly believes to lead to a high-value state. Since the agent observes transitions from the environment in RL, it sees that the action doesn't lead to a high-value state, and has the ability to correct this behavior instead of blindly continuing the original learned strategy. In general, there exist a breadth of strategies for changing behaviors based on witnessed environment transitions that are untapped by contemporary offline RL algorithms, which learn static policies using objectives not designed to incentivize adaptation.

The Bayesian perspective defines an optimal adaptive offline RL policy, which is the solution to a POMDP implicitly induced by the offline dataset \citep{Duff2002OptimalLC, Ghosh2021WhyGI}. Unfortunately, this optimal policy cannot be easily obtained using standard model-free POMDP algorithms, since the POMDP is only an \textit{implicit construction} defined by the agent's epistemic uncertainty. Rather, we derive an approximate approach to solve the POMDP that uses an ensemble of value functions to estimate the implicit partial observability, and learns an adaptive policy using this ensemble. By optimizing this adaptive policy with the implicit POMDP objective, the policy learns to adapt in a way that helps the agent maximize evaluation performance in the true environment.

The primary contribution of our work is to formally demonstrate the necessity of adaptation in offline RL, and to provide a practical algorithm for learning optimally adaptive policies. We use the Bayesian formalism to show why adaptability is necessary, and how a policy must adapt to optimally maximize offline RL performance. As a first step towards approximating Bayes-optimal adaptiveness, we propose an ensemble-based offline RL algorithm that imbues policies with the ability to adapt within an episode, and demonstrate its favorable properties in offline RL domains.

\section{Problem Setup}
The goal in reinforcement learning is to maximize an agent's expected return in a Markov decision process (MDP), written as $\gM \coloneqq (\gS, \gA, r, T, \rho, \gamma)$. An MDP is specified by a state space $\gS$, action space $\gA$, Markovian transition kernel $T(s'|s, a)$, reward function $r(s,a)$, initial state distribution $\rho(s_0)$, and discount factor $\gamma$. A history is a sequence of witnessed transitions in an MDP $h_t = (s_0, a_0, r_0, s_1, a_1, \dots, s_t)$; we write $s(h_t) \coloneqq s_t$ to denote the current state.  A policy $\pi$ maps a history to a distribution over actions, and achieves return in the MDP $J_{\gM}(\pi) = \E_\pi[\sum_{t \geq 0} \gamma^t r(s_t, a_t)]$. The maximal return in an MDP can be attained by a Markovian policy (one that depends on $h_t$ only through $s_t$) \citep{puterman1994markov}.
 
The value function of a policy $\pi$ is the expected return of the policy having taken an action $a$ after witnessing a history $h$,  $Q_{\gM}^\pi(h_t, a_t)=\E_{\pi}[\sum_{t'=0}^\infty \gamma^{t'} r(s_{t+t'}, a_{t+t'}) \vert h, a]$. The gradient of $J_\gM(\pi_\theta)$ can be written using $Q^\pi$:
\begin{equation}
\label{eq:pg_mdp}
    \nabla_\theta J_{\gM}(\pi_\theta) = \E_{h \sim \pi}[ \nabla_\theta \E_{a \sim \pi_\theta(\cdot | h)}[Q_{\gM}^\pi(h,a)]].
\end{equation}
This expression serves as the basis for the policy objective in many standard actor-critic algorithms for MDPs. While these updates are traditionally written for Markovian policies $\pi_\theta(a|s_t)$, we write the more general form using history-based policies $\pi_\theta(a | h_t)$ since our paper studies the use of history-based policies in an MDP when doing offline RL.

In offline RL, the agent receives a dataset of trajectories $\gD = \{\tau_{i}\}_{i=1}^n$ pre-collected by a behavior policy $\pi_\beta$ from some MDP $\gM^*$, and must use this dataset to learn a policy that performs well in this MDP. The performance of an offline RL policy is the expected return it achieves in the true MDP: $J_{\gM^*}(\pi)$. In this paper, we will show that even though there exist optimal Markov policies in an MDP, when learning from an offline dataset, history-based policies $\pi_\theta(a | h_t)$ can play a role in improving performance. We refer to history-based policies as \textit{adaptive} policies interchangeably throughout the text.

\section{When History is Useful for Offline RL}
\label{sec:examples}
\begin{figure}
    \centering
    \includegraphics[width=0.9\linewidth]{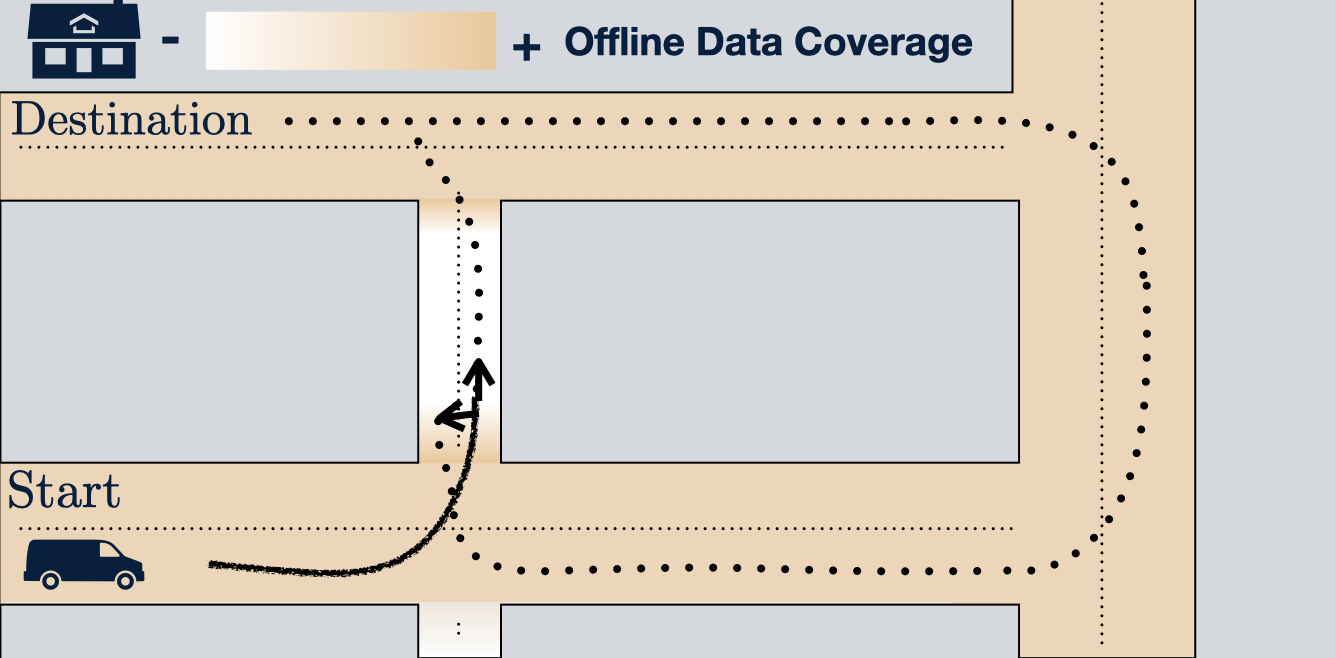}
    \vspace{-1em}
    \caption{\footnotesize The \textit{City navigation} example, where there is low data coverage on the small side street. An adaptive policy, which first tries the side street and reverts to the large road if unknown circumstances arise, will outperform Markovian alternatives, such as the conservative solution that always takes the larger road. }
    \label{fig:navigation}
    \vspace{-1em}
\end{figure}
Before formally studying optimality in offline RL, we illustrate example offline RL problems where adaptive policies achieve higher return than pessimistic Markov policies. In the first example, pessimism will provide a sub-optimal solution that can be improved with adaptation; in the second, pessimism will completely fail to solve the problem, making adaptation necessary to achieve high performance.

\textbf{City navigation:} Suppose we wish to learn to navigate in a city to a destination as quickly as possible from an offline dataset of trajectories where data is mostly concentrated on big city roads and highways, but a few trajectories navigate smaller side streets. Given the dataset, the agent has less uncertainty about paths relying on the big roads, but there exist shorter paths using side streets for which the agent has high uncertainty (visualized in Figure \ref{fig:navigation}). An algorithm that trains an RL algorithm using this dataset as a replay buffer would take the side streets, since empirically this leads to the goal fastest. This has a potentially high downside: if the agent encounters unforeseen conditions on the side streets, it may be unable to reach the goal by a reasonable time. 

A pessimistic offline RL algorithm sidesteps this failure mode by penalizing the side streets for having high uncertainty, since they have low coverage in the dataset, and learns to navigate only on big roads. While robust, this penalty prevents capitalization on the upside of the side street potentially leading to the destination faster. If we step beyond Markovian policies, we can learn adaptive strategies that can capitalize on this upside without the failure mode of standard RL: for example by starting along the side-street, and doubling back to the main road if it witnesses unexpected conditions along the way. Adaptation improves over a pessimistic solution because at test-time, the environment-agent feedback loop can signal whether the risky behavior (taking the side street) may be successful, enabling revision and consequent improvement within an evaluation episode.

\textbf{Locked doors:} Consider an environment where an agent is asked to exit a room with four doors, one unlocked and three locked. Each episode, a different door is left unlocked, and the agent receives an image whose label indicates the unlocked door (hence the environment is fully observable). Given access to a limited set of training images in the offline dataset, offline RL must learn a policy that can parse a new image at evaluation and go through the correct door. Any standard offline RL method (including pessimistic approaches)
will learn a Markovian policy that goes towards the door that the agent believes has the highest chance of being unlocked, but such strategies are sub-optimal for this setting. When such a policy encounters an image for which its original prediction is incorrect, the agent will try to open a locked door. Since the agent's policy does not change and the agent remains in the same state, it will simply attempt the same action again, and become stuck perpetually trying to open the wrong door. The only way to avoid getting stuck is to try a different door, whether by policy adaptation or some stochastic posterior sampling scheme. 

In both examples, the offline dataset left uncertainty about how to act optimally, and adaptability allowed the agent to outperform than the fixed strategy learned by pessimism. In the next section, we will study optimality in offline RL to formalize this intuition, towards understanding why adaptability can improve the performance of offline RL agents.

\begin{figure}
    \centering
    \includegraphics[width=0.7\linewidth]{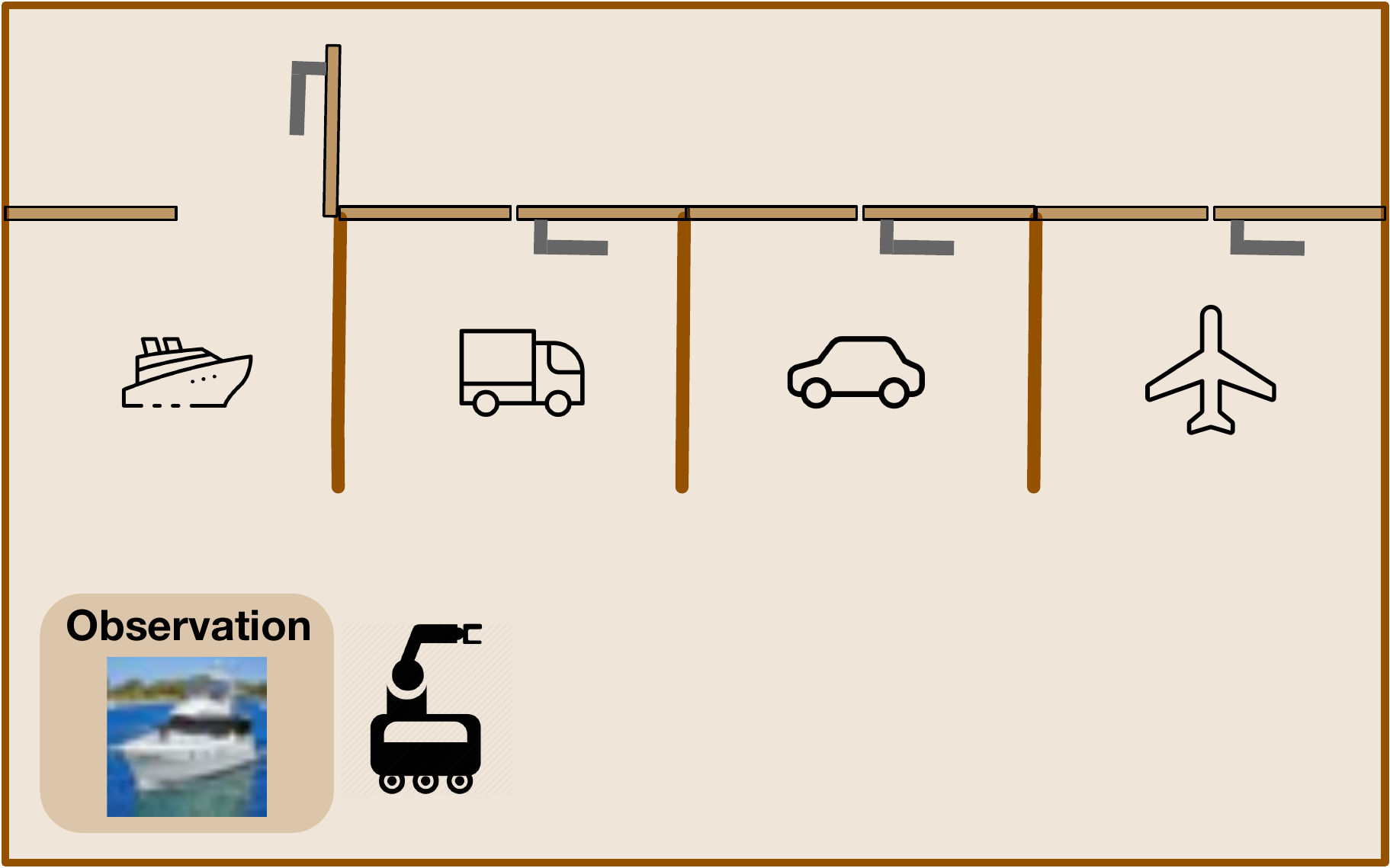}
    \vspace{-0.5em}
    \caption{\footnotesize The \textit{locked doors} task with CIFAR-10 images. Adaptive policies, which can change which door to try based on environment feedback, will outperform Markovian policies, which by design, continue trying the same incorrect door for the entire episode. }
        \vspace{-1em}
    \label{fig:locked_doors_desc}
\end{figure}

\section{Adaptive Policies are Optimal for Offline RL}

We now analyze offline RL from a Bayesian perspective, to examine \textit{why} adaptive behavior might be useful in offline RL, and \textit{how} an offline RL agent should adapt to be maximally performant.

During training, an offline RL agent only receives information about the true environment $\gM^*$ via the dataset $\gD$. In general, this dataset does not uniquely identify $\gM^*$: if the data-collection policy only visits a limited subset of the state space or only takes certain actions, there may be many potential MDPs that behave identically on the states and actions in the dataset, but differ in their transitions and rewards on out-of-sample states and actions. As a result, the dataset induces epistemic uncertainty about the identity of the MDP.
Formally, the dataset $\gD$ and a prior distribution over MDPs $P(\gM)$ define a \emph{posterior distribution} over MDPs, 
$\epomdp \propto P(\gM) P(\gD | \gM)$.

After training, the policy learned via offline RL is deployed into the true MDP, and must act to maximize return in this environment.
The Bayesian objective for policy learning under this model is to maximize return in expectation over MDPs from this posterior distribution,
\begin{equation}
\vspace{-0.5em}
\label{eq:epistemic_pomdp}
    J_{\text{Bayes}}(\pi) \coloneqq \E_{\gM \sim \epomdp}[J_{\gM}(\pi)],
\end{equation}
since this objective defines the agent's expected performance in the true environment under its prior beliefs. As is common in Bayesian treatments of RL \citep{Ghavamzadeh2015BayesianRL}, we call Equation \ref{eq:epistemic_pomdp} the Bayesian offline RL objective, and define its maximizers as \textit{Bayes-optimal policies} for offline RL.

To shed light on what maximizing the Bayesian objective requires, we note that Equation \ref{eq:epistemic_pomdp} corresponds to the expected return of a policy in a partially observable environment defined by the agent's epistemic uncertainty called the \textbf{epistemic POMDP}, following a construction in a Bayesian treatment of generalization in RL \citep{Ghosh2021WhyGI}. 

The epistemic POMDP $\gM_{po}$ may be described as follows: in each episode, the agent is placed in a new MDP $\gM \sim \epomdp$ and asked to maximize reward, but is not told the identity of this MDP. As the MDP $\gM$ dictates the environment dynamics for the agent, but its identity omitted from the observation vector, the agent faces a partially observable problem.
This construction mirrors that of the epistemic POMDP in \citet{Ghosh2021WhyGI} and Bayes-adaptive MDPs \citep{Duff2002OptimalLC}, so we leave the complete technical specification for Appendix~\ref{appendix:pomdp_formulate}. Reflecting on the city navigation example,  we might imagine that the posterior distribution consists of two MDPs, one where the side street has low traffic and one where it has high traffic. Acting optimally without knowing which of these MDPs the agent was placed in requires adaptation, since midway through an episode, the agent will learn whether the side street has high or low traffic and thereby resolve its partial observability.

We can use the epistemic POMDP perspective on Equation \ref{eq:epistemic_pomdp} to show that acting Bayes-optimally for offline RL in an MDP \textit{requires} adaptation, and that in general,  pessimistic Markov policies may do an arbitrarily poor job of optimizing the offline RL objective. 
\begin{theorem}[informal]
The optimal solution to the Bayesian offline RL objective is, in general, adaptive. Further, there are offline RL problem instances $(\gD, p(\gM))$ where the best Markovian policy far underperforms the optimal adaptive solution.
\end{theorem}

These claims are formally stated and proven in Appendix~\ref{appendix:pomdp_formulate}, but we provide intuition here. Since Equation \ref{eq:epistemic_pomdp} can be interpreted as the expected return in a POMDP, Bayes-optimality for offline RL is equivalent to acting optimally under partial observability, which necessitates adaptivity \cite{Monahan2007ASO}. This partial observability also explains why pessimism cannot by itself lead to optimal performance for offline RL. An agent that cannot change its behavior on receiving new information that resolves its partial observability, even if  pessimistic, will be unable to keep up with an adaptive agent that reacts and adapts to this signal.

It is important to note that the epistemic POMDP is not a physical environment that the agent will explicitly ever act in, but rather only an interpretation of the Bayesian offline RL objective. That is, the agent never explicitly encounters partial observability (it always acts in an MDP); rather, we are using the parlance of partial observability to describe how uncertainty induced by the offline dataset affects policy learning and evaluation process under a Bayesian viewpoint.

The POMDP reinterpretation of the Bayesian offline RL objective provides concrete benefits for designing offline RL algorithms. First, it demonstrates the necessity of adaptation in offline RL, since non-Markovianity is well-known to be required for handling partial observability. It also provides post-facto justification for stochastic policies and stochastic regularization, which are common algorithmic design choices in offline RL~\citep{fujimoto2019off, kostrikov2021fisher},
since stochastic policies are better than deterministic counterparts under partial observability \cite{eysenbachMaxEnt, singhPOMDP}. %
We will see in the next section that although it may be infeasible to directly apply POMDP algorithms to learn Bayes-optimal offline RL policies, they can nonetheless serve as inspiration for designing adaptive offline RL methods.

\section{Optimizing for Adaptation in Offline RL}

Our analysis establishes the need for adaptive policies in offline RL, and prescribes performance in the epistemic POMDP (Equation \ref{eq:epistemic_pomdp}) as an objective to learn optimal test-time adaptation mechanisms. Unfortunately, the straightforward approach to learning such policies, explicitly modeling the epistemic POMDP and running a generic POMDP algorithm, is unscalable: it requires modeling the posterior distribution over environment dynamics and reward $\epomdp$, which cannot be obtained with high fidelity in most applications of interest.

We derive a policy update that avoids the burden of learning a posterior over environment dynamics by instead modeling the simpler posterior over value functions $P(Q_\gM^\pi | \gD)$, using the fact that the value function $Q_{\gM}^\pi$ entangles the necessary information about both dynamics and rewards for a given policy. This reduces the burden of uncertainty estimation, as we no longer need to learn an uncertainty-aware dynamics model, and also allows us to leverage recent progress in learning value function posteriors in offline RL~\citep{an2021uncertainty,anonymous2022why}.In this section, we define a suitable class of adaptive policies and derive an exact policy update for the objective, which will serve as the backbone for the practical actor-critic algorithm that we will introduce in Section \ref{sec:algo}.

\subsection{Uncertainty-Adaptive Policies}

In the previous section, we showed the sub-optimality of Markovian policies $\pi(\cdot | s)$ for offline RL since they lack the ability to change behavior on receiving new information that resolves its epistemic uncertainty. The following proposition characterizes precisely what Markovian policies are missing: a measure of how the agent's uncertainty has changed since the beginning of the episode. 

\begin{definition}
The relative MDP belief $\vb(h)$ for a history $h$ is the relative change in the posterior distribution over MDPs after incorporating the history:%
\begin{equation}
\vb(h)(\gM) = \frac{P(\gM | h, \gD)}{P(\gM | \gD)}.
\end{equation} %
\end{definition}

\begin{restatable}{prop}{beliefpolicyoptimal}
\label{corr:beliefpolicyoptimal}
The Bayesian offline RL objective is maximized by a policy depending only on the current state and relative MDP belief: $\pi(\cdot | h_t) = \pi(\cdot | s_t, \vb(h_t))$.
\end{restatable}

The relative MDP belief $\vb(h)$ should be interpreted as a \textit{weighting} that indicates which of the original hypotheses generated during offline training are consistent with the trajectory seen so far during evaluation. In the epistemic POMDP, it serves as the belief state, serving as a compressed statistic of all the information in the history before the current state that is relevant to maximizing future return. 

This proposition implies that when optimizing the offline RL objective, it is sufficient to consider only \textbf{uncertainty-adaptive} policies $\pi(\cdot | s, \vb)$, those that conditions their behavior not only on the current state (as a Markovian policy does), but also on a sufficient statistic of the relative MDP belief. The simplicity of uncertainty-adaptive policies is particularly appealing: in structure, it is a Markovian policy that takes in an additional input $\vb$, but because this new input $\vb$ changes throughout the episode, the policy is no longer static but rather adaptive with history. Another useful property of uncertainty-adaptive policies is that they are allowed to adapt and change behavior only if the agent's epistemic uncertainty about the environment changes.

\subsection{The Bayesian offline RL policy gradient}

With an appropriate class of adaptive policies, we turn to optimizing the Bayesian objective. We consider a standard first-order optimization approach, whereby a parametric stochastic uncertainty-adaptive policy $\pi_\theta(\cdot | s, \vb)$ updates the parameters $\theta$ according to the gradient of the Bayesian objective (equivalently, the epistemic POMDP objective). The following proposition shows that this policy gradient can be written in terms of the posterior distribution of the current policy's value function.

\begin{restatable}{prop}{bayesgradient}
\label{prop:bayes_gradient}
The gradient of the Bayesian offline RL objective for an uncertainty-adaptive policy $\pi_\theta(\cdot | s, \vb)$ can be written in terms of the MDP value functions $Q_\gM^\pi$ from the posterior distribution:
\begin{align}
\label{eq:bayes_gradient}
 \nabla_\theta J_{\text{Bayes}}(\pi_\theta) &= \E_{h \sim \pi}[ \nabla_\theta \E_{a \sim \pi_\theta(\cdot | s(h), \vb(h))}[ \nonumber\\&{\color{purple} \E_{\gM \sim P(\gM | \gD)}[} {\color{teal}\vb(h)(\gM)}Q_\gM^\pi(h, a){\color{purple}]}].
\end{align}
\end{restatable}
It is instructive to compare the policy gradient of the Bayesian offline RL objective to a standard MDP policy gradient (Equation \ref{eq:pg_mdp}), accented in color above. First, since the exact MDP is not known from the dataset (only a posterior), we must average the value functions across the induced posterior distribution ${\color{purple}P(\gM | \gD)}$. Second, since the agent's uncertainty changes during an episode, this average is re-weighted by ${\color{teal}\vb(h)(\gM)}$ to account for how the agent's uncertainty has changed from the original dataset posterior after witnessing history $h$. To complete the specification of the policy update, we describe the value function for an uncertainty-adaptive policy $\pi$ in an MDP $\gM$.

\begin{restatable}{prop}{valuefnpomdp}
\label{prop:value_fn_epistemic_pomdp}
The value function for a uncertainty-adaptive policy in an MDP $Q_\gM^\pi(h, a)$ depends on $h_t$ only through $(s_t, \vb(h_t))$ and satisfies the following Bellman recursion:
\begin{equation}
\label{eq:value_consistency}
    \begin{split}
    Q_{\gM}^\pi(s, \vb, a) &= r(s, a) + \gamma \E_{\substack{s' \sim \gM \\ a \sim \pi}} \left[Q_{\gM}^\pi(s', \vb', a)\right]
    \end{split}
\end{equation}
where $\vb' \coloneqq \operatorname{BeliefUpdate}(\vb, (s, a, r, s))$ is the new relative MDP belief after witnessing $(s, a, r, s')$,
\begin{equation}
     \operatorname{BeliefUpdate}(\vb, (s, a, r, s))(\gM) \propto p_\gM(r, s' | s, a) \vb(\gM)
\end{equation}
\end{restatable}

The value function for an uncertainty-adaptive policy is similar to that of a Markovian policy, the main change being the need to take into account how the relative MDP belief changes over time  ($\vb$ on the left, $\vb'$ on the right). Since the value function understands how values change when the agent's uncertainty changes and the policy adapts, optimizing the policy using these value functions provides the learning signal for the policy to adapt in a way most conducive to maximizing test-time return.

Together, the propositions describe policy and value function learning rules needed to learn Bayes-optimal uncertainty-adaptive policies for offline RL. Holistically, three main characteristics distinguishes this Bayesian offline RL policy update from policy optimization in a known MDP: 1) we must learn a posterior distribution over value functions instead of a single value function,  2) the policy conditioned on a belief $\vb$ must maximize the $\vb$-reweighted average across the value function posterior, and 3) the value function posterior must be trained to account for how the agent's uncertainty changes when a new transition is seen.

\begin{algorithm*}
\caption{Adaptive Policies with Ensembles of Value Functions (\algoname)}\label{alg:true_epi}
\begin{algorithmic}
\STATE Receive input: dataset $\gD$, number of ensemble members $n$
\STATE Initialize policy $\pi(\cdot | s, \vb): \gS \times \dist_n \to \dist(\gA)$
\STATE Initialize ensemble of value functions $\{\hQ_1, \dots \hQ_n\}$, where $\hQ_k(s, \vb, a): \gS \times \dist_n \times \gA \to \R$ 
\WHILE{$\pi$ has not converged}
    \STATE Sample transition $(s, a, r, s') \sim \gD$ from dataset and possible belief $\vb \sim p(\vb)$
    \STATE Approximate next-step belief $\vb' = \text{BeliefUpdate}(\vb, (s, a, r, s'))$ using Equation \ref{eq:surrogate_belief_update}

    \STATE Optimize value functions to minimize TD error taking into account the updated belief $\vb \to \vb'$
    \begin{equation}
        \min \gL(\hQ_k) \coloneqq (\hQ_k(s, \vb, a) - (r + \gamma \E_{a' \sim \pi(\cdot | s', \vb')}[\hQ_k(s', \vb', a')]|))^2 ~~~\forall k \in \{1, \dots n\} 
    \vspace{-0.2em}
    \end{equation}

     \STATE Optimize adaptive policy $\pi(\cdot | s, \vb)$ to maximize $\vb$-weighted average of value functions
       \begin{equation}\label{eq:true_h}
        \max_{\pi(\cdot | s, \vb)} \E_{a_\pi \sim \pi}[\sum_{k} \vb_k \hQ_k(s, \vb , a_\pi)]
        \end{equation}
    \vspace{-0.5em}
\ENDWHILE
\end{algorithmic}
\end{algorithm*}

\section{Practical Algorithm}
\label{sec:algo}

In order to practically implement the policy update dictated in the previous section, we must approximate the posterior distribution over value functions for our current policy, and choose a representation for the relative MDP belief to pass into our adaptive policy. This section tackles these issues; the result is an actor-critic algorithm, \algoname, for learning adaptive policies for offline RL in an MDP.

\subsection{A finite approximation to the posterior}

We choose to approximate the posterior distribution over value functions given the offline dataset  $P(Q_\gM^\pi | \gD)$  with a finite ensemble of value functions $\{\hQ_1, \dots \hQ_n\}$, where $\hQ_k$ is the value function of $\pi$ in $\gM_k \sim p(\gM | \gD)$. For this finite posterior, the relative MDP belief $\vb$ is a probability distribution over the $n$ MDPs corresponding to our value function ensemble, starting as a uniform distribution $\vb_0 = [\frac{1}{n} \dots \frac{1}{n}]^\top$ at the beginning of the episode, and skewing towards the MDP most consistent with the witnessed trajectory as the episode continues. In the remainder of this section, we will write $\vb \in \R^n$ on the $n$-dimensional probability simplex, with $\vb_k \coloneqq \vb(\gM_k)$.

\subsection{Training value functions}

We now discuss how each value function $\hQ_k$ in the posterior should be learned. As described by Proposition \ref{prop:value_fn_epistemic_pomdp}, learning the value function for a uncertainty-adaptive policy $\pi(a | s, \vb)$ requires two changes from that for Markov policies: 1) the value function must take in both the state $s$ and relative belief $\vb$, since the policy $\pi$ depends on both; 2) the Bellman consistency equation must incorporate how the relative belief shifts $\vb \to \vb'$ when the new transition $(s, a, r, s')$ is witnessed.

Ensuring that the value function accounts for how new transition $(s, a, r, s')$ changes the relative belief is what provides the policy the signal for adjusting its adaptation mechanism towards higher performance. One problem with implementing this scheme is that $\operatorname{BeliefUpdate}(\vb, (s, a, r, s))$ depends on transition log-likelihoods $\log P_{\gM}(s', r | s, a)$, which we do not model, and so we must replace it with an approximation. Recognizing that $- \log P_{\gM}(s', r | s, a)$ represents the information-theoretic surprise of witnessing the transition $(s, a, r, s')$ in $\gM$, we choose to replace it with a surrogate surprise $\log \hat{P}_{\gM} (\cdot)$ defined by our value model:
\begin{equation}
    \label{eq:surrogate_belief_update}
    \begin{split}
    \log \hat{P}_{\gM_k}&(s', r | s, a) = -\big|\hQ_k(s, \vb, a) \\
     &- (r + \gamma \E_{a' \sim \pi}[\hQ_k(s', \vb, a')])\big|^2 
    \end{split}
\end{equation}
$\log \hat{P}_{\gM_k}$ provides a measure of how likely we are to see the tuple $(V(s'), r)$ in $\gM_k$ (as a proxy for the likelihood of $(s', r)$) and remains aligned with the state dynamics model $P_{\gM_k}$, in that both are high if $(s', r)$ truly comes from $\gM_k$. 

With these pieces in place, to update $\hQ_k(s, \vb, a)$ using the transition $(s, a, r, s')$, we first estimate the new belief vector $\vb' = \operatorname{BeliefUpdate}(\vb, (s, a, r, s))$ using our surrogate model, and then optimize to achieve the consistency in Equation \ref{eq:value_consistency} by minimizing TD error.
This leads to the following objective for a single value function in our ensemble:
\vspace{-0.5em}
\begin{equation}
\begin{split}
    \gL_{\text{critic}}(\hQ_k) &= \E_{\substack{(s, a, r, s') \sim \gD \\ \vb \sim p(\vb)}}\big[\big(\hQ_k(s, \vb, a) \\
    & - (r + \gamma \E_{a' \sim \pi(\cdot | s', \vb')}[\hQ_k(s', \vb', a')]|)\big)^2\big].
\end{split}\nonumber
\end{equation}\par
    \vspace{-1em}
Sampling the relative weighting $\vb \sim p(\vb)$ independently of the transition  allows us to estimate values for all possible belief weightings that we may encounter at any state.

\subsection{Training the policy}
The policy we learn takes in a state $s$ and current belief vector $\vb$ and outputs a distribution over actions. As discussed in Proposition \ref{prop:bayes_gradient}, for any state $s$ and belief $\vb$, the optimal ascent direction maximizes the average over the value functions in our ensemble \textit{weighted by the belief}, leading to the following policy loss for an offline actor-critic method:
\begin{equation}
    \gL_{\text{actor}}(\pi) = -\E_{\substack{s \sim \gD \\ \vb \sim p(\vb)}}\big[\E_{a \sim \pi_\theta(\cdot |s, \vb)}[\sum_k \vb_k \hQ_{k}(s, \vb,a)]\big].\nonumber
\end{equation}

This policy objective has a simple interpretation: since $\vb$ defines our current belief over which MDP we are actually in, when choosing an action to take at a particular state $s$, we should place greater consideration on the value function for the MDPs that we consider to be more likely. As in the value function objective, we train the policy for many different choices of belief so that at test-time, our policy can act appropriately for whatever belief we may have at that particular moment. 
\begin{algorithm}
\caption{\algoname ~Test-Time Adaptation}\label{alg:true_epi_test}
\begin{algorithmic}
\STATE $s_0 = $ \textsc{env.reset()}
\STATE Initialize belief vector to uniform: $\vb_0 = [\frac1n, \dots, \frac1n]^\top$
\FOR{ environment step $t=0,1, \dots$}
    \STATE Sample action: $a_t \sim \pi(\cdot| s_t, \vb_t)$
    \STATE Act in environment: $r_t, s_{t+1} \gets$ \textsc{env.step($a_t$)}
    \STATE Update belief vector using new transition (Eq \ref{eq:surrogate_belief_update})\\
    ~~~~~~~~~~$\vb_{t+1} = \operatorname{BeliefUpdate}(\vb_t, (s_t, a_t, r_t, s_{t+1}))$
\ENDFOR
\end{algorithmic}
\end{algorithm}

\vspace{-1em}
\subsection*{Summary}

Our actor-critic algorithm (outlined in Algorithm \ref{alg:true_epi}) has three core components: an ensemble of state-action value functions $\{\hQ_1, \dots \hQ_n\}$, a weighting $\vb(h)$ over the ensemble that represents how well each value function describes the seen history, and an adaptive policy $\pi(\cdot | s, \vb)$ trained against the ensemble. For any belief weighting $\vb$, the policy is trained to maximize the $\vb$-weighted average of the value function ensemble, which prioritizes the value functions most consistent with the seen history. The ensemble of value functions is trained using standard TD methods, but modified to account for how the belief weighting changes $\vb \to \vb'$ with new transitions. During evaluation, the belief state $\vb$ is updated online using Equation \ref{eq:surrogate_belief_update}, which downweights ensemble members whose value predictions are inconsistent with the new transitions; this, in turn, causes the policy to output actions focused on maximizing only the consistent value functions in the ensemble.

\section{Related Work}

\textbf{Offline RL:}   
Policy learning challenges in offline RL \citep{Lange2012BatchRL, levine2020offline} arise when the environment is underspecified, for instance when the dataset is collected with a narrow behavior policy or collected only in part of the state spaces~\citep{mandlekar2021matters, Fu2020D4RLDF}. A commonly used principle to handle this is pessimism: biasing the policy learning objective towards the data-collection policy, and disincentivizing behaviors that may take the agent away from the distribution of states seen in the dataset. Pessimism can be incorporated in many ways, such as policy regularization~\citep{fujimoto2019off, kumar2019stabilizing, ghasemipour2021emaq, kostrikov2021offline} or forming conservative value estimates~\citep{kumar2020conservative, an2021uncertainty, luo2021mesa, jin2021pessimism}. In practice, these methods can suffer from over-conservatism without careful tuning ~\citep{kumar2021workflow}.

\textbf{Ensembles in RL:} Ensembles of value functions have been used in the online RL setting for a variety of purposes: to avoid over-estimation \cite{Hasselt2016DeepRL, Fujimoto2018AddressingFA}, to enable faster exploration \cite{Osband2013MoreER}, and to stabilize the training objective \cite{Chen2021RandomizedED, Lee2021SUNRISEAS}. In offline RL, value ensembles have been deployed primarily to provide a source of pessimism into the policy optimization objective: \citet{Agarwal2020AnOP} averages an ensemble of value functions to attain more stable target values in TD-learning, \citet{anonymous2022why} form LCB estimates of expected return in the MDP using an ensemble of value functions; \citet{an2021uncertainty} optimizes policies against the most pessimistic value function in the ensemble. In contrast to these static methods that learn a Markovian policy optimizing some ensemble statistic (e.g. min, LCB, median).  our adaptive method learns an weighting-conditioned policy that optimizes all possible re-weightings of the ensemble, and uses Bayes filtering at test-time to recover the best re-weighting.

\textbf{Bayesian RL:} We study offline RL within the Bayesian RL framework (see \citep{Ghavamzadeh2015BayesianRL} for a survey), which has been studied in many other sub-fields of RL \citep{Ramachandran2007BayesianIR, Lazaric2010BayesianMR, Jeon2018ABA, Zintgraf2020VariBADAV}, most prominently for deriving exploration strategies in online RL. The epistemic POMDP interpretation of the Bayesian offline RL objective we introduce is a specific instantiation of the Bayes-adaptive MDP \citep{Duff2002OptimalLC} and related to the construction of \citet{Ghosh2021WhyGI}. Being exactly Bayes-optimal is known to be intractable, so algorithms have been developed to learn approximations of Bayes-optimality that leverage either explicit access to the posterior over MDPs or trajectories collected in the posterior \cite{Zintgraf2020VariBADAV, Dorfman2020OfflineML}. This prevents their immediate application in offline RL, since this posterior distribution cannot be constructed with high fidelity, motivating our approach of avoiding explicitly modelling this posterior over dynamics.

\section{Experiments}

The primary aim of our experiments is to ascertain whether adaptability leads to improved performance in offline RL. Thus, we provide an evaluation on standard D4RL benchmark tasks \cite{Fu2020D4RLDF} and two offline RL tasks that require handling ambiguity and generalization, \textit{Locked Doors} and \textit{Procgen Mazes}. We do not expect adaptability to uniformly improve across these domains, since some datasets and environments do not lead to multiple hypotheses that an agent may adapt between (e.g., in the city navigation example, if the dataset contained \textit{no} trajectories from small roads, then no improvement can be expected over the conservative strategy using big roads). Implementation details and hyperparameters in Appendix~\ref{appendix:locked_doors}.

\begin{figure}
    \centering
    \includegraphics[width=0.9\linewidth]{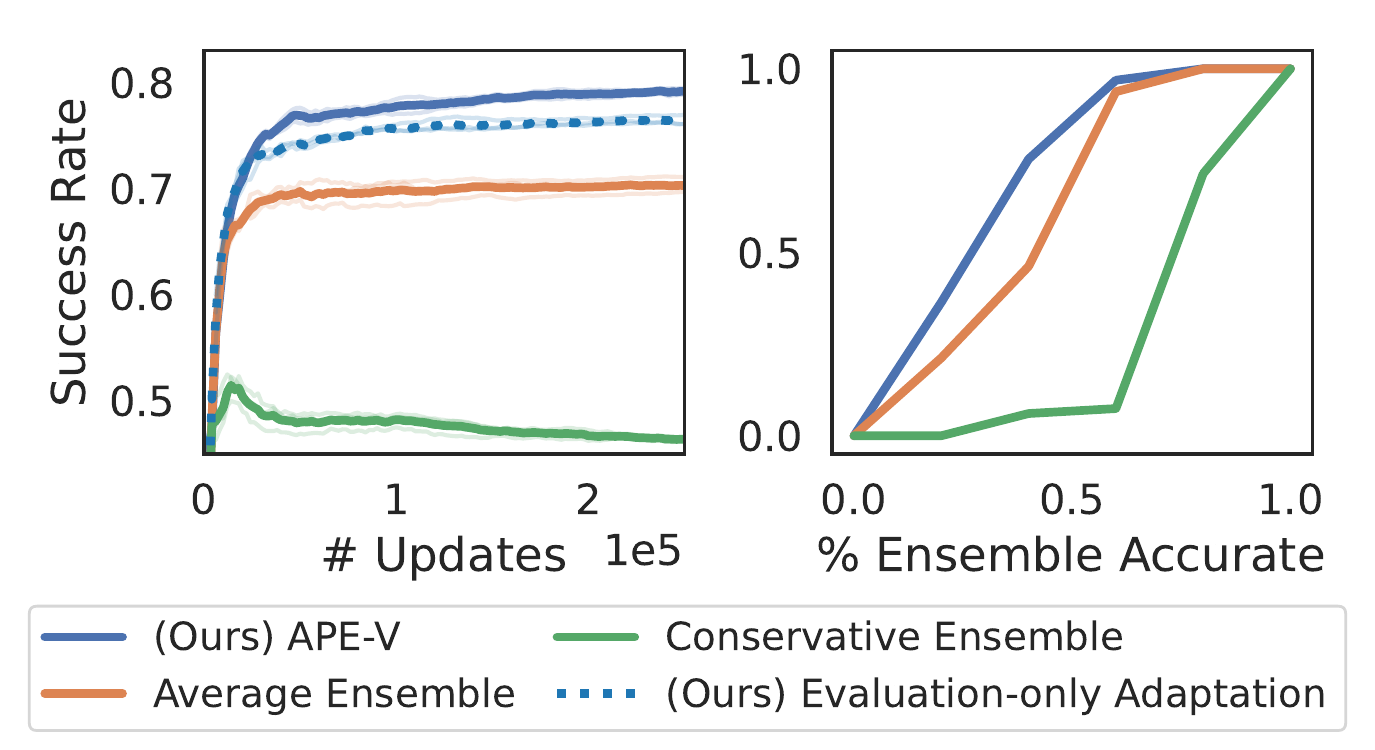}
    \vspace{-1.5em}
    \caption{\textbf{Locked Doors with CIFAR-10:} (left) Success rate over training for various ensemble-based offline RL schemes (right) Success rate conditioned on the number of ensemble members with successful classification. Adaptive policies outperform stochastic and pessimistic alternatives because they succeed with higher probability if at least one ensemble member outputs accurate values.}
    \vspace{-1em}
    \label{fig:locked_doors_success}
\end{figure}
\vspace{-0.5em}
\begin{table*}
	\centering
	\normalsize
	\caption{Normalized average returns on D4RL suite, averaged over 4 seeds. Full results in Appendix \ref{appendix:d4rl}.}
	\vspace{0.2em}
	\label{tab:gym}
	\begin{adjustbox}{max width=\linewidth}
		\begin{tabular}{l|cccccc|r}
			\toprule
			\multirow{2}{*}{\textbf{Task Name}} & \multirow{2}{*}{\textbf{BC}} & 
			\textbf{SAC} & 
			\textbf{REM} & 
			\textbf{CQL} & 
			\textbf{IQL} & 
			\textbf{SAC-$N$} & 
			\multirow{2}{*}{\textbf{\algoname}} \\
			& & \citep{haarnoja2018soft} & \citep{Agarwal2020AnOP} & \citep{kumar2020conservative} & \citep{kostrikov2021offline} & \citep{an2021uncertainty} & \\
			\midrule
			halfcheetah-medium & 43.2$\pm$0.6 & 55.2$\pm$27.8 & -0.8$\pm$1.3 & 44.4 & 47.4$\pm$0.2 & 67.5$\pm$1.2 & \textbf{69.1 $\pm$ 0.4} \\
			halfcheetah-medium-replay & 37.6$\pm$2.1 & 0.8$\pm$1.0 & 6.6$\pm$11.0 & 46.2 & 44.2$\pm$1.2 & \textbf{63.9$\pm$0.8} & \textbf{64.6 $\pm$ 0.9}  \\
			hopper-medium-replay & 16.6$\pm$4.8 & 7.4$\pm$0.5 & 27.5$\pm$15.2 & 48.6 & 94.7$\pm$8.6 & \textbf{101.8$\pm$0.5} & {98.5 $\pm$ 0.5} \\
			walker2d-medium & 70.9$\pm$11.0 & -0.3$\pm$0.2 & 0.2$\pm$0.7 & 74.5 & 78.3$\pm$8.7 & 87.9$\pm$0.2 & \textbf{90.3 $\pm$ 1.6} \\
			walker2d-medium-replay & 20.3$\pm$9.8 & -0.4$\pm$0.3 & 12.5$\pm$6.2 & 32.6 & 73.8$\pm$7.1 & 78.7$\pm$0.7 & \textbf{82.9 $\pm$ 0.4}\\
\bottomrule
		\end{tabular}
	\end{adjustbox}
	\vspace{-1em}
\end{table*}

\subsection{Locked Doors with CIFAR10}

We first study the \textit{Locked Doors} domain from Section \ref{sec:examples}, where an agent seeks to exit a room, but must infer which door it can leave out of by parsing a CIFAR-10 image  (visualized in Figure \ref{fig:locked_doors_example_trajectory}, details in Appendix~\ref{appendix:locked_doors}). Since CIFAR-10 is a well-studied problem in supervised classification where the training set does not fully eliminate uncertainty about test image labels, embedding CIFAR-10 into an offline RL navigation problem provides us a controlled way of studying the effect of a limited offline dataset within an RL domain with a challenging perception component.  %

Figure \ref{fig:locked_doors_success} displays the success rate of policies learned by various offline RL procedures, where we see that policies that are adapted at test-time  outperform non-adaptive baselines. Using \algoname, which trains the value ensemble and policy to approximate Bayes-optimal behavior, leads to the highest performance and robust adaptation. Using Equation \ref{eq:surrogate_belief_update} to adapt a policy from an ensemble trained agnostic of adaptation also leads to large improvements, although less than when trained explicitly for adaptation with \algoname. We observe qualitatively that the learned \algoname~policy better takes advantage of the spatial layout of the room than an adaptive agent not trained for adaptation; the agent tries and eliminates doors physically closer, rather than in order of most-likely to least-likely, leading to more efficient pathing.

To test the resiliency of these ensemble-based agents, we plot their success rate conditioned on how many value functions provide correct predictions for the given image (Figure \ref{fig:locked_doors_success} (right)).
This scales poorly for a conservative policy trained with LCB values, since an incorrect value function causes actions towards the correct door to be excessively penalized. Averaging models performs better, but degrades once the majority of the ensemble members are incorrect. In comparison, the adaptive policy can often succeed in an environment even when only one ensemble member offers correct predictions for the current image. We provide implementation details for all comparisons in Appendix~\ref{appendix:locked_doors}.

\begin{figure}
    \centering
    \includegraphics[width=0.8\linewidth]{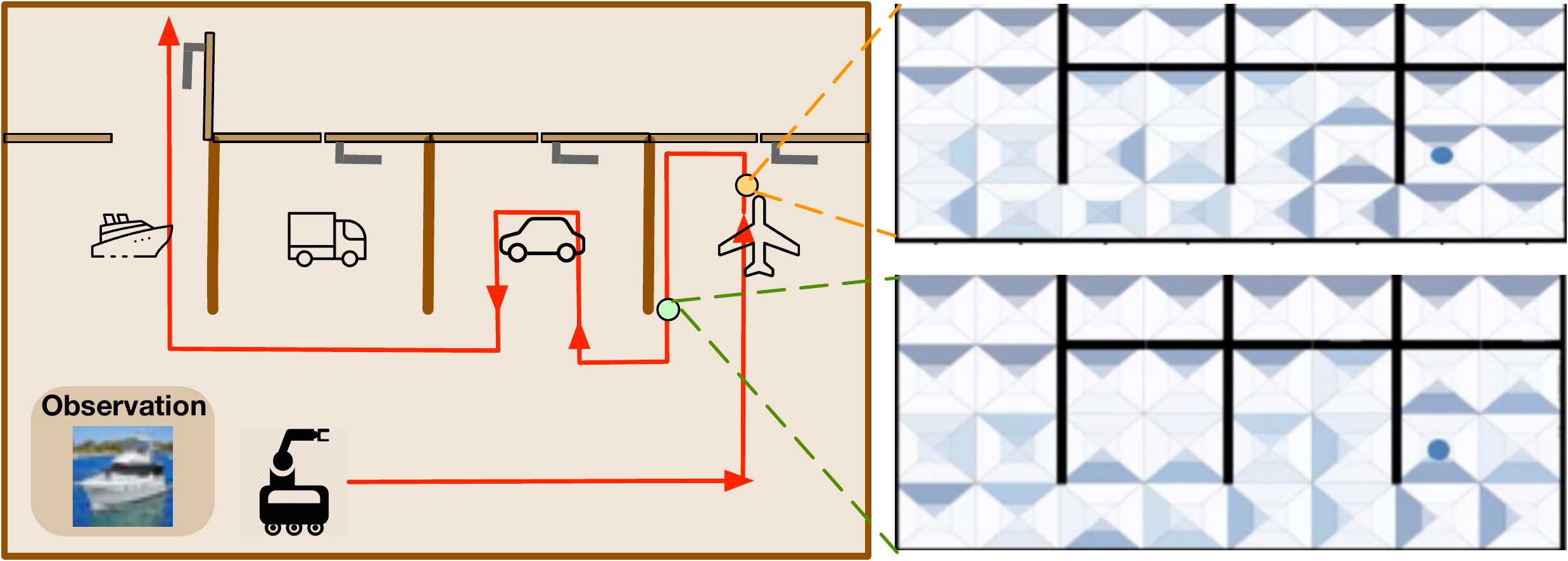}
    \vspace{-1em}
    \caption{\textbf{Adaptation in Locked Doors:} (left) A trajectory within the \textit{Locked Doors} domain; (right): visualization of adaptive policy before and after trying the locked door on the far right.  }
    \vspace{-1em}
    \label{fig:locked_doors_example_trajectory}
\end{figure}

\subsection{Procgen Mazes}
We next investigate the performance of APE-V on an offline variant of the Maze task from the Procgen benchmark \citep{Cobbe2020LeveragingPG}, a challenging image-based benchmark. An agent, which receives a top-down $64 \times 64$ image of the maze (visualized in Figure \ref{fig:procgen_performance}), must navigate to the specified exit before the episode ends. During training, the agent receives an offline dataset of transitions from $N_{\text{train}}$ training mazes, each with differing layout and visual textures, and is evaluated on new unseen mazes during deployment. We construct the offline dataset to contain transitions of the agent at all locations within the training mazes -- note that despite this uniform coverage, the agent faces epistemic uncertainty at test-time since it is evaluated on new levels never seen in the offline dataset. The full experimental setup is described in Appendix \ref{appendix:procgen}.

We compare the performance of APE-V to an ablation that doesn't account for epistemic uncertainty and one that learns conservative value functions (using LCB) rather than being adaptive.When provided a training dataset of 1000 training mazes ($\approx~5 \times 10^5$ transitions), APE-V is able to reliably solve $79\%$ of new mazes at test-time, higher than that achieved by conservatism or ignoring uncertainty (Figure \ref{fig:procgen_performance}). In a more data-limited setting with only 200 training mazes, APE-V is also able to outperform the conservative approach, although all methods succeed far less frequently on held-out mazes than the previous setting (Table \ref{tab:procgen}).    
\begin{figure}
    \centering
    \includegraphics[width=0.38\linewidth]{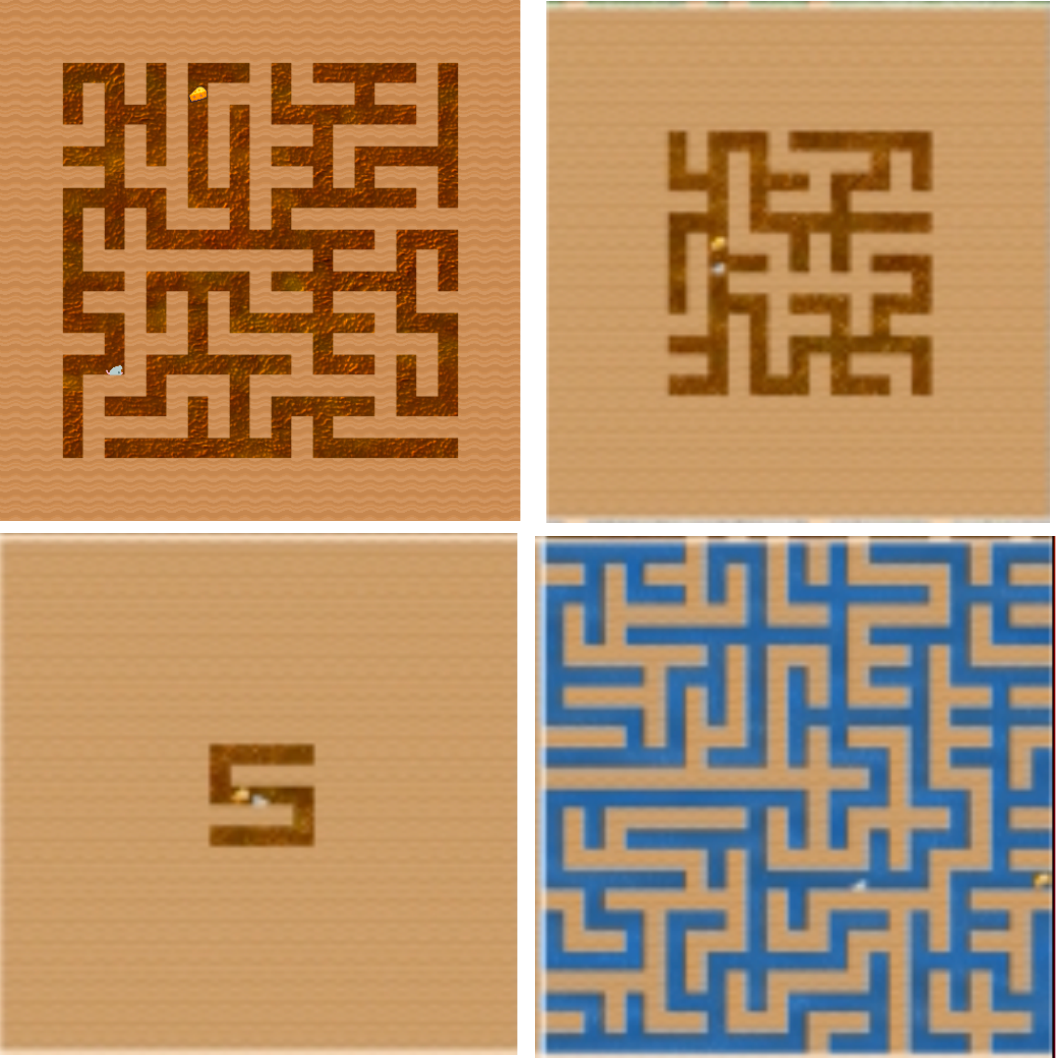}
    \includegraphics[width=0.6\linewidth]{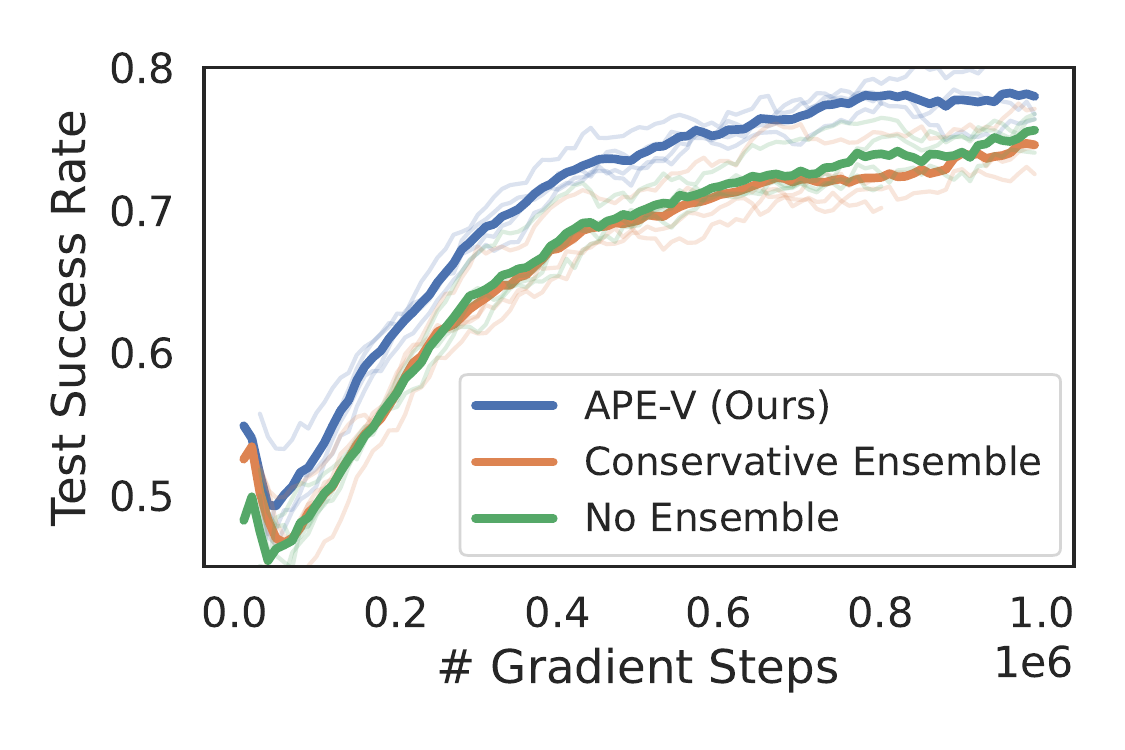}
    \vspace{-1em}
    \caption{\textbf{Procgen Mazes:} (left) 4 sample mazes visualized in Procgen. (right) Evaluation success rates on held-out mazes when trained on offline dataset of 1000 mazes. APE-V, which adaptively chooses between candidate value functions, performs better than acting conservatively with respect to the value ensemble. }
    \vspace{-1em}
    \label{fig:procgen_performance}
\end{figure}

\begin{table}
	\centering
	\caption{Maze-solving success rates, averaged over 4 seeds.}
	\small
	\vspace{0.2em}
	\label{tab:procgen}
	\begin{adjustbox}{max width=\linewidth}
	\begin{tabular}{l|cc|cc}
			\toprule
			 &
			 \multicolumn{2}{|c|}{\textbf{200 Train Levels}}
             & 
                \multicolumn{2}{|c}{\textbf{1000 Train Levels}}\\
                & Train & Test & Train & Test \\ 
			\midrule
			No Ensemble & 0.96 {\tiny $\pm 0.02$}& 0.24 {\tiny $\pm .02$}& 0.93 {\tiny $\pm 0.01$} & 0.75 {\tiny $\pm 0.03$}\\
			Conservative & 0.96 {\tiny $\pm 0.02$} & 0.23 {\tiny $\pm 0.02$}& 0.93 {\tiny $\pm 0.01$} & 0.75 {\tiny $\pm 0.04$} \\
			APE-V &  0.97 {\tiny $\pm 0.00$} & \textbf{0.31} {\tiny $\pm 0.04$}& 0.92 {\tiny $\pm 0.01$}& \textbf{0.79} {\tiny $\pm 0.04$}\\
\bottomrule
		\end{tabular}
	\end{adjustbox}
\end{table}

\subsection{D4RL}
To complement our analysis in the CIFAR-10 Locked Rooms domain and Procgen Mazes, we also investigate the performance of \algoname~on the D4RL benchmark \citep{Fu2020D4RLDF}. We instantiate our method using SAC-$n$ \citep{an2021uncertainty} as the base procedure to learn each value function in our ensemble (details in Appendix~\ref{appendix:d4rl}).

Within-episode adaptation with \algoname~leads to higher evaluation performance than SAC-$n$, which learns a pessimistic Markov policy, and other pessimism-focused methods like CQL \citep{kumar2020conservative} when trained on data from a medium-performance agent (\texttt{xxxx-medium}), or data from the buffer of an RL agent (\texttt{xxxx-medium-replay}). When the data is collected from a random policy (\texttt{xxxx-random}) or contains expert demonstrations (\texttt{xxxx-medium-expert}), \algoname~performs equivalently to the conservative baseline We do note that the improvement from adaptivity on D4RL is relatively minor; we believe this happens because these tasks generally do not have data distributions that lead to multiple salient hypotheses, so adapting amongst these strategies leads to marginal gain. 

To better understand adaptation in this domain, we compare the performance of the adaptive policy to non-adaptive policies that holds the belief state static to $\vb = \ve_k$ (i.e., following only the $k$-th ensemble member $\hQ_k$), visualized for \texttt{walker2d-medium} in Figure \ref{fig:adaptation_d4rl}. For this task, we see that the different members of the ensemble lead to slightly differing performances, and that \algoname~ in fact receives average return above all these individual strategies, indicating that adaptation within the episode indeed allows the policy to adapt to a better strategy than it may have started with.

\begin{figure}
    \centering
    \includegraphics[width=0.9\linewidth]{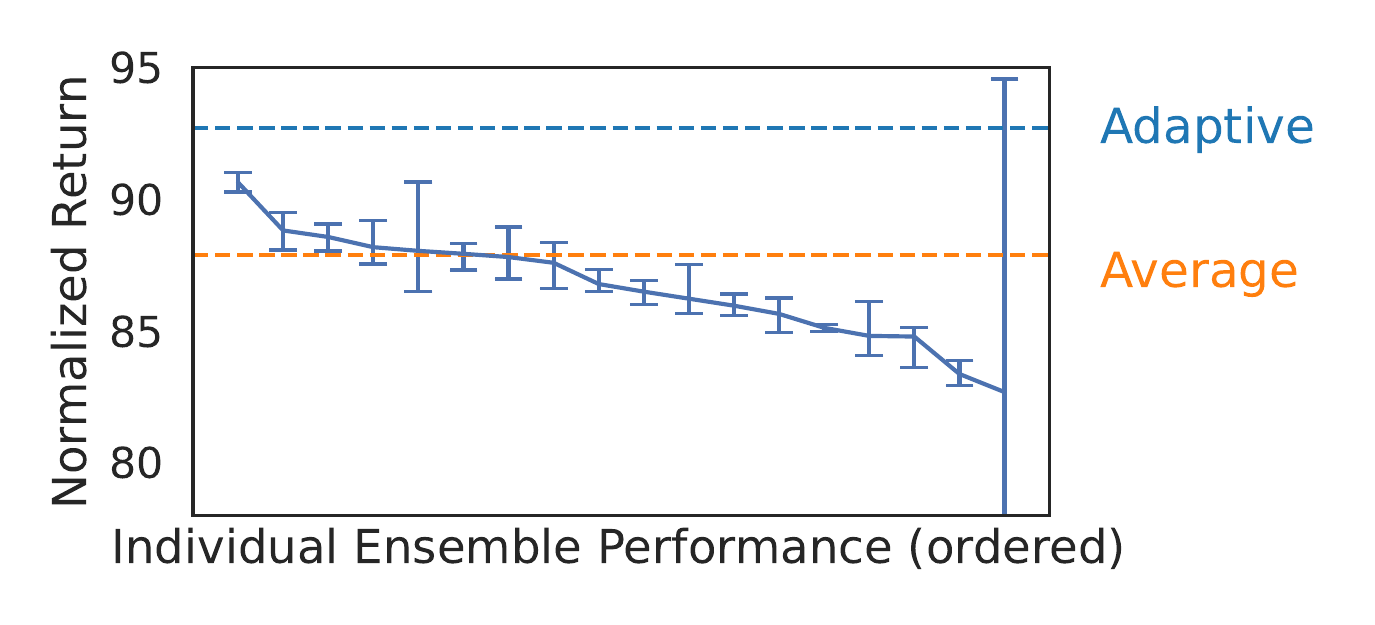}
    \label{fig:adaptation_d4rl}
    \vspace{-1em}
    \caption{\textbf{Adaptation in D4RL:} Performance of ensemble members vs adaptive policy on \texttt{walker2d-medium}. Adapting the belief $\vb$ within an episode outperforms any static choice of $\vb$.}
    \vspace{-1em}
\end{figure}

\section{Discussion}

In this paper, we discussed how the Bayesian perspective on offline RL naturally leads to adaptive policies. We showed that the epistemic uncertainty an offline RL agent faces due to the limited dataset manifests as an implicit partial observability, and therefore necessitates adaptive behaviors to act optimally. We then derived a policy gradient formulation in this setting that allows us to express the gradient of an adaptive policy in terms of a distribution over Q-functions and a posterior over MDPs that is induced by the uncertainty inherent in any finite-data offline RL problem. We instantiated an offline RL algorithm based on this principle called \algoname, and showed how an ensemble of value functions can be used to approximate the theoretically motivated adaptive policy update. Our algorithm is only a first foray in optimizing for adaptation in offline RL. Methods with more expressive posterior distributions over value functions, or those with more scalable adaptation mechanisms have the potential for greater benefits in adaptation. Moving forward, understanding how we may infuse existing pessimistic perspectives in offline RL with the benefits conferred by adaptivity is likely to be an exciting direction, one which we hope will lead to more powerful algorithms for learning behaviors from offline sources. 

\subsection*{Acknowledgements}
The authors thank Abhishek Gupta, Aviral Kumar, Colin Li, Katie Kang, Laura Smith, Young Geng, the members of RAIL \& Improbable AI Lab, and the anonymous reviewers for discussions and helpful feedback. We thank MIT Supercloud and the Lincoln Laboratory Supercomputing Center for providing compute resources. This research was supported by an NSF graduate fellowship, a DARPA Machine Common Sense grant, an MIT-IBM grant, the Office of Naval Research, ARL W911NF-21-1-0097, and Intel.

\bibliography{example_paper}
\bibliographystyle{icml2022}

\newpage
\appendix
\onecolumn

\section{Supplementary for Section 4}
\label{appendix:pomdp_formulate}

\subsection{The Epistemic POMDP Formalism}

\textbf{POMDPs:} Before we outline the epistemic POMDP, we first provide a quick introduction to partially observable MDPs (POMDPs). A POMDP is defined by the tuple $(\bar{\gS}, \gA, \gO, \bar{P}, O, r, \rho, \gamma)$, where $\bar{\gS}$ is the (hidden) state space of the POMDP, $\gA$ the action space, $\gO$ is the observation space for the agent, $\bar{P}(\bar{s}' | \bar{s},a )$ is a Markovian transition function between hidden states, $O$ is the emission function that maps a hidden state to the agent's observation $o_t = O(\bar{s}_t)$, $r(\bar{s},a)$ , $\rho(\bar{s})$ the initial hidden state distribution, and $\gamma$ discount factor. The belief state for a POMDP is given by $p(\bar{s} | h)$, the mapping from a history to the conditional distribution over hidden states having seen this history. It is well-known that the optimal policy in a POMDP is in general history-dependent, and that there is always an optimal policy that depends on history only through the belief state. 

\textbf{The Epistemic POMDP: }The epistemic POMDP for offline RL $\gM_{po}$ is a POMDP that satisfies the following property: for any policy $\pi$, $J_{\gM_{po}}(\pi) = J_{\text{Bayes}}(\pi)$. This property means that optimizing the Bayesian offline RL objective is equivalent to optimization in the epistemic POMDP, a useful equivalence for deriving properties of optimality in offline RL and designing new offline RL algorithms.

Formally, given a posterior distribution $\epomdp$ over MDPs with shared state space $\gS$ and action space $\gA$, $\gM_{po}$ is defined as following. The hidden state space is $\bar{\gS} = \gS \times \mathbf{M}$ and a state represented as $\bar{s} \coloneqq (s, \gM)$ and the action space is still $\gA$, the transition function extended as $\bar{P}((s', \gM') | (s, \gM), a) = \delta(\gM = \gM')P_{\gM}(s' | s, a)$. The observation function omits the identity of $\gM$ from the agent observation: $O((s, \gM)) = s$. The reward function is $r((s, \gM), a) = r_\gM(s, a)$ and initial state distribution $\rho((s, \gM)) = \epomdp \rho_{\gM}(s)$.

\subsection{Proof of Theorem 4.1}
\label{app:thm41}
We note that many prior works in POMDPs \citep{singhPOMDP, Duff2002OptimalLC} that show that the optimal solution for POMDPs in general, and for Bayesian objectives is adaptive. In the following theorem, we construct an instance of the epistemic POMDP in which the adaptive policy significantly outperforms all Markovian policies, addressing both desired components.
\begin{proposition}[Sub-optimality of Markovian policies and optimality of adaptiveness]
Let $n \in \mathbb{N}$. There are offline RL problem instances $(\gD, p(\gM))$ with $n$-state MDPs where the adaptive Bayes-optimal policy achieves $J_{Bayes}(\pi_{\text{adaptive}}^*) = -2n$ but the highest performing Markovian policy achieves return of a magnitude worse: $J_{Bayes}(\pi_{\text{markov}}^*) \leq -\frac{1}{2}n^2$.
\end{proposition}

\textbf{Summary and Intuition:} Adaptivity is particularly useful when an agent needs to try a sequence of actions, and if it fails, then to backtrack on its steps and return to try something else -- behavior that a Markovian policy cannot well imitate. We capture this intuition in our construction, a chain environment where the agent is uncertain if the goal is on the left or the right, so performing well under the Bayesian objective requires being able to reach both the left-hand side of the environment and the right-hand side of the environment. We will show that an adaptive agent can do this in $O(n)$ timesteps but an Markovian policy will need at least $O(n^2)$ timesteps.

\begin{figure}[b]
    \centering
    \includegraphics[width=0.6\linewidth]{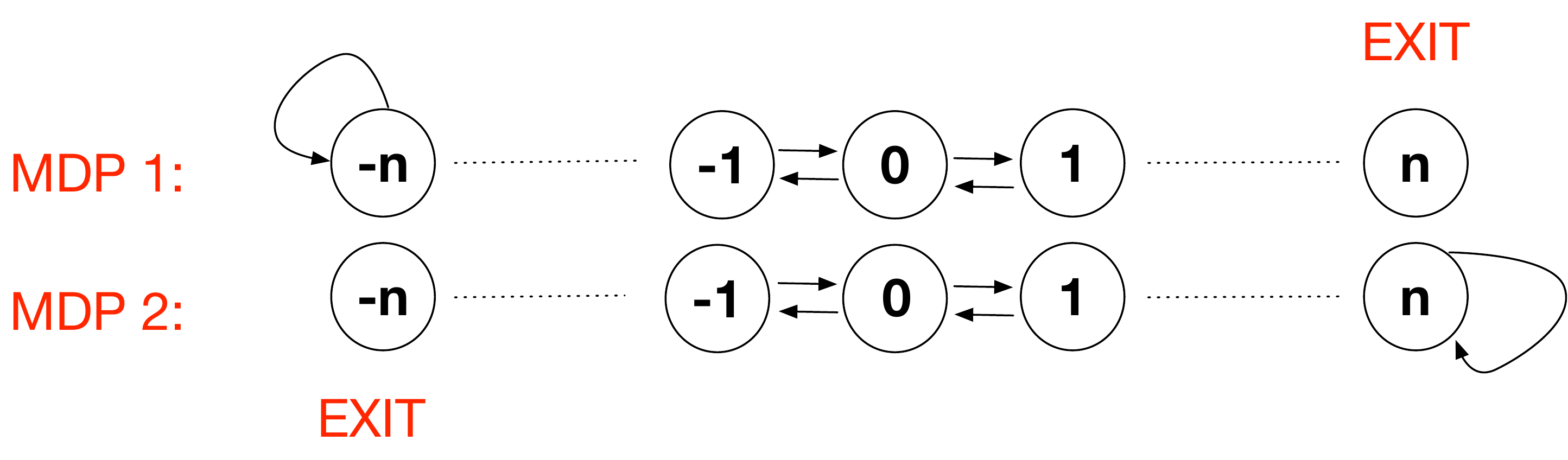}
    \vspace{-0.5em}
    \caption{\footnotesize Visualization of MDP construction described in Appendix~\ref{app:thm41}}
    \label{fig:mdp_construction_proof}
\end{figure}

\begin{proof}~\\
\textbf{MDP construction:} Our construction uses two standard chain MDPs $\gM_1, \gM_2$ supported on the states $\{-n, \dots, 0, \dots, n\}$ where the agent starts at state $0$; at all states $k$, the agent can go left (to $k-1$) or right (to $k+1$), as long as these states exist and receives $-1$ reward at every timestep. We modify these MDPs as follows: in the first MDP $\gM_1$, entering state $n$ leads to immediate termination; in the second MDP $\gM_2$, entering state $-n$ leads to immediate termination. Let $\gamma = 1$.

\textbf{Construction of Posterior Distribution: }Suppose that after having received an offline dataset that the posterior distribution was uniform on these MDPs: $p(\gM_1 | \gD) = p(\gM_2 | \gD) = \frac{1}{2}$. One way that this could happen in practice is if the prior $p(\gM)$ had $p(\gM_1) = p(\gM_2) = \frac{1}{2}$, and no trajectories in the dataset actually reached the ends of the chain ($n$ or $-n$) so it received no signal about this ambiguity. The value of the Bayesian offline RL objective for this problem for a policy is given by $J_{Bayes}(\pi) = -\frac{1}{2}(T_\pi(n) + T_\pi(-n)$, where $T_\pi(n)$ is the expected amount of time it takes $\pi$ to first reach state $n$ and $T_\pi(-n)$ defined similarly.

\textbf{Performance of Bayes-optimal adaptive policy: }The optimal adaptive policy for the Bayesian offline RL objective defined by $p(\gM | \gD)$ is simple: it first goes right for $n$ timesteps, and then left for $2n$ timesteps (or vice-versa). If it is in $\gM_1$, then it exits in $n$ timesteps, otherwise in $3n$ timesteps. This leads to a return of 
\[J_{Bayes}(\pi_{\text{adaptive}}^*) = \frac{1}{2}(-n) + \frac{1}{2}(-3n) = -2n\]

\textbf{Upper bound on performance of Markovian policy: } Consider a Markovian policy $\pi(a|n)$ acting in this environment and define $p_n = \pi(\text{right} | n)$, $q_n = 1- p_n$, and $\ell_n = \frac{p_n}{q_n}$. Notice that $T_\pi(n) + T_\pi(-n)$ can be lower-bounded by $T_{n, 0}$ and $T_{-n, 0}$, where $T_{n, 0}$ is the expected time of $\pi$ to first return to $0$ after visiting $n$, and $T_{-n, 0}$ equivalently. We suppose without loss of generality that $T_{n, 0} < T_{-n, 0}$. We can write \[T_{n, 0} = \underbrace{\E[\text{Time to reach n from 0}]}_{T_{0 \to n}} + \underbrace{\E[\text{Time to reach 0 from n}]}_{T_{n \to 0}}.\]

Noticing that the agent's movement is described by a birth-death chain with birth rate $p_n$ and death rate $q_n$, we can use results about mean absorbtion time in birth-death chains from the stochastic processes literature \citep{2020Discrete} to receive lower-bounds on $T_{0 \to n}$ and $T_{n \to 0}$:

\begin{align}
    T_{0 \to n} &\geq \sum_{j=0}^{n-1} \sum_{k=j+1}^{n-1} \prod_{i=j}^k \ell_i\\
    T_{n \to 0} &= \sum_{j=0}^{n-1} \sum_{k=j+1}^{n-1} \prod_{i=j}^k \ell_i^{-1}\\
    \intertext{We can combine the two to get a lower bound on $T_{n, 0} = T_{n \to 0} + T_{0 \to n}$}
    T_{n, 0} = T_{n \to 0} + T_{0 \to n} &\geq  \sum_{j=0}^{n-1} \sum_{k=j+1}^{n-1} \left(\prod_{i=j}^k \ell_i\right) + \left(\prod_{i=j}^k \ell_i\right)^{-1}\\
\end{align}
We see that this lower bound on $ T_{0 \to n} +  T_{n \to 0}$ is minimized when $\ell_i = 1$ (that $p_i = q_i = \frac{1}{2}$) for all states $i$, and takes value $n^2$. This provides a lower bound on $T_\pi(n) + T_\pi(-n)$ and therefore, the following upper bound on the performance of the Markovian policy:
\[J_{Bayes}(\pi_{\text{markov}}^*) \leq -\frac{1}{2}n^2\]
\end{proof}
\clearpage
\section{Derivations for Section 5}
\label{appendix:adapt_formulate}

\beliefpolicyoptimal*

\begin{proof}
We will show that there is an optimal policy that takes form $\pi(\cdot | s, \vb)$ by showing that $(s, \vb)$ forms a belief state for the epistemic POMDP, and use the well-known fact that there always exists an optimal policy depending only on the belief state of the POMDP \citep{Monahan2007ASO}.

Recall that the true state in the epistemic POMDP $\bar{s}$ is given by the tuple $\bar{s} \coloneqq (s, \gM)$, where $s$ is the current MDP state and $\gM$ the MDP currently being acted in. By definition, a belief state for the POMDP is one that is isomorphic to the distribution $P(\bar{s} | h)$.

We first show that $P(\bar{s} | h)$ can be recovered from $(s(h), \vb(h))$:
\begin{align*}
P_{\gM_{po}}(\bar{s}|h) = P_{\gM_{po}}(s, \gM|h) &=  P_{\gM_{po}}(s|h) P_{\gM_{po}}(\gM|h, s)\\
&= \delta(s = s(h)) P(\gM | h, \gD) \\
&= \delta(s = s(h))\vb(h)(\gM) P(\gM | \gD).
\end{align*}
That $(s(h), \vb(h))$ can be recovered from $P(\bar{s} | h)$ follows immediately from the definition of $\vb(h)$ (since it is defined in terms of $P(\gM | h, \gD)$). Therefore, $(s(h), \vb(h))$ is a belief state for the epistemic POMDP, and as immediate corollary, there exists an optimal policy for the epistemic POMDP that depends only on this belief state.
\end{proof}

\bayesgradient*
\begin{proof}
From the equivalence $J_{\text{Bayes}}(\pi_\theta) = J_{\gM_{po}}(\pi_\theta)$ and the policy gradient of a history-based policy in a POMDP \cite{Monahan2007ASO}, we have that
\begin{align}
    \nabla_\theta J_{\text{Bayes}}(\pi_\theta) = \E_{h \sim \pi}[ \nabla_\theta \E_{a \sim \pi_\theta(\cdot | h)}[Q_{\gM_{po}}^\pi(h, a)]],
\end{align}
where $Q_{\gM_{po}}^\pi(h, a)$ is the value function of $\pi$ in the epistemic POMDP. To prove the desired statement, we show that $Q_{\gM_{po}}^\pi(h, a) = \E_{\gM \sim \epomdp}[\vb(h)(\gM)Q_\gM^\pi(h, a)]$.

Let $p_{\gM}^\pi(\tau)$ be the distribution over trajectories $(s_0, a_0, r_0, s_1, a_1, \dots)$ for the belief-based policy $\pi$ in MDP $\gM$. From construction of the epistemic POMDP, we can see that $p_{\gM_{po}}^\pi(\tau) = \E_{\gM \sim \epomdp}[p_\gM(\tau)]$. For a $T$-step history $h \coloneqq (s_0, a_0, r_0, s_1, \dots s_T)$, the value function of $\pi$, $Q_{\gM_{po}}^\pi(h, a)$ is given by 
\begin{align}
    Q_{\gM_{po}}^\pi(h, a) &= \E_{\tau \sim p_{\gM_{po}}^\pi}\left[\sum_{t=T}^{\infty} \gamma^{t - T} r_{t} | h=h, a=a\right]\\
    &= \int \left(\sum_{t=T}^{\infty} \gamma^{t - T} r_{t}\right) p_{\gM_{po}}^\pi(\tau | h) d\tau\\
    &= \int\int \left(\sum_{t=T}^{\infty} \gamma^{t - T} r_{t}\right) p_{\gM}^\pi(\tau | h) P(\gM | \gD, h) d\gM~d\tau\\
    &= \int\int \left(\sum_{t=T}^{\infty} \gamma^{t - T} r_{t}\right) p_{\gM}^\pi(\tau | h) \vb(h)(\gM)P(\gM | \gD) d\gM~d\tau\\
    &= \int\int \left(\sum_{t=T}^{\infty} \gamma^{t - T} r_{t}\right) p_{\gM}^\pi(\tau | h) \vb(h)(\gM)P(\gM | \gD) d\gM~d\tau\\
    &= \int \underbrace{\E_{\tau \sim p_\gM^\pi}\left[ \left(\sum_{t=T}^{\infty} \gamma^{t - T} r_{t}\right) | h=h, a=a \right]}_{Q_\gM^\pi(h, a)} \vb(h)(\gM)P(\gM | \gD) d\gM~d\tau\\
    &= \E_{\gM \sim P(\gM | \gD)}[\vb(h) Q_\gM^\pi(h, a)]
\end{align}
\end{proof}

\valuefnpomdp*
\begin{proof}
Consider an arbitrary $T$-step history $h \coloneqq (s_0, a_0, r_0, s_1, \dots, s_T)$; we write $\vb_t$ to be the relative MDP belief after $t$ steps. The value of any policy $\pi$ in $\gM$ after taking action $a$ is defined as 
\[Q_\gM^\pi(h, a) = \E_{\pi, \gM}[\sum_{t=T}^\infty \gamma^{t-T} r_t | h=h, a_T=a].\]

We note the following conditional independences: $s_{T+1} \perp h ~|~ \left(s_T, a_T\right)$ because $\gM$ has Markovian transition dynamics and that $a_{T+1} \perp h~|~\left(s_T, \vb_T\right)$ as $a_{T+1}$ depends only on $s_{T+1}$ and $\vb_{T+1}$ (since $\pi$ is belief-based) and $\vb_{T+1}$ depends only on $s_T$, $a_T$, and $\vb_T$. Iterating the argument over, we see that the future trajectory depends on $h$ only through $s_T$ and $\vb_T$, and therefore the expected discounted return $Q_\gM^\pi(h, a)$ also only depends on $h$ through $s_T$ and $\vb_T$. 

 \[Q_\gM^\pi(h, a) = Q_\gM^\pi(s(h), \vb(h), a).\]
 
 To derive the recursion, we simply notice that by definition, $Q_\gM^\pi(h, a)$ must satisfy 
 \begin{align}
     Q_\gM^\pi(h, a) = r(s, a) + \gamma \E_{s' \sim P_\gM(\cdot | s, a)}[\E_{a' \sim \pi(\cdot | h \cup (s, a, r, s')}[Q_\gM^\pi(h \cup (s, a, r, s'), a')]]\\
    \intertext{Injecting the fact that $Q_\gM^\pi$ and $\pi$ depend only on $s$ and $\vb$, we get the desired consistency equation}
     Q_\gM^\pi(s, \vb, a) = r(s, a) + \gamma \E_{s' \sim P_\gM(\cdot | s, a)}[\E_{a' \sim \pi(\cdot | s', \vb')}[Q_\gM^\pi(s', \vb', a')]]
\end{align}

\end{proof}

\clearpage
\section{Details for \textit{Locked Doors} Domain}
\label{appendix:locked_doors}
\subsection{Task Description}
At beginning of every episode, an image is uniformly sampled from a reduced CIFAR10 dataset consisting only of classes airplane, automobile, ship and truck. The agent receives this ${32 \times 32 \times 3}$ image in its observation, alongside its current $(x, y)$ position in the room (see figure \ref{fig:locked_doors_desc}); success requires the agent to go through the door corresponding to the right image class within $T=50$ steps; it physically is unable to go through any of the other doors during the episode. The action space is discrete corresponding to the four cardinal directions.

\subsection{Implementation details}
For the locked doors domain, we instantiate \algoname~using deep Q learning~\citep{mnih2013playing} as the base learning algorithm to train the ensemble of Q networks. Since the action-space is discrete, we do not maintain a separate actor network, instead directly computing the optimal policy given the value functions. For \algoname, this corresponds to using $\pi(a|s, \vb) = \argmax_{a} \sum_k \vb_k \hQ_k(s, \vb, a)$ throughout training and test-time. For the average ensemble baseline, we train the $Q$ networks separately using deep Q learning, and only combine the value functions together at test-time. For \algoname~(evaluation only) baselines, we follow the same training procedure as for the average ensemble baseline, but choose actions according to the adaptive policy $\argmax_{a} \vb_k \hQ_k(s, \vb, a)$ at test-time. Our conservative ensemble baseline uses an LCB estimate of Q-values to define the policy ($\pi(a|s) = \argmax \text{mean}(\{\hQ_k(s,a)\}_k) -  \beta \text{std}(\{\hQ_k(s,a)\}_k)$) during training and evaluation.

\begin{table}[t!]
    \centering
    \small
    \caption{Hyperparameters used for training Q learning based agents in Locked Doors domain}
    \begin{tabular}{cc}
    \toprule
    Hyperparameter & Value \\ [0.5ex]
    \midrule
     $\gamma$ & 0.98 \\ [0.5ex]
     batch size & 256 \\ [0.5ex]
     learning rate & 1e-3 \\ [0.5ex]
     Optimizer & Adam~\citep{kingma2014adam} \\ [0.5ex]
     Training steps & 250k \\ [0.5ex]
     Number of ensembles & 5 \\ [0.5ex]
     $p(\vb)$ & SymmetricDirichlet(0.1) \\ [0.5ex]
    \bottomrule
    \end{tabular}
    \label{tbl:hyperparams_locked}
\end{table}
Hyperparameters for training \algoname~and the accompanying baselines are presented in table~\ref{tbl:hyperparams_locked}. The neural network architecture we use for our value functions has a CNN head used for processing the CIFAR image, with 3 convolution layers with output channel of 32 and kernel size of 3 and a dense layer with output dimension of 10. Each of the convolution layer is followed by ReLU activation and Avg pooling with a stride of 2. The output of the CNN head, the current belief vector $\vb$ and the (x,y) position of the agent are concatenated and passed to a fully connected network, with 2 hidden layers of size 256 and ReLU activation, to get Q values for all the 4 actions.   

\subsection{Additional Analysis} 
\begin{wrapfigure}{r}{0.3\linewidth}
\includegraphics[width=\linewidth]{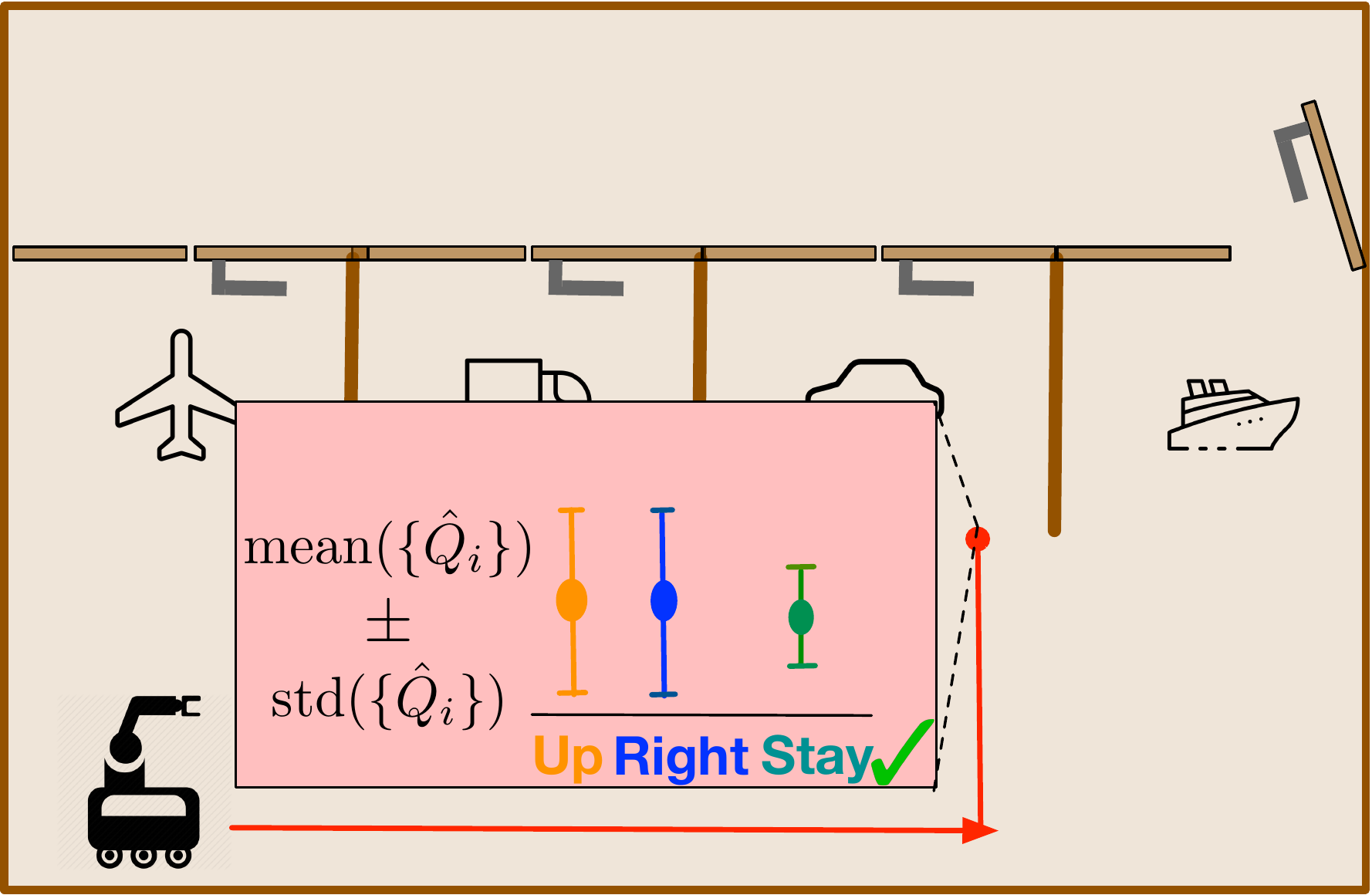}
\end{wrapfigure}
In our experiments, we found that the conservative ensemble significantly underperforms, receiving lower return than even a single ensemble member. When visualizing the behavior of the conservative baseline, we found that it has a tendency to get stuck in place even before trying any of the doors. As an example of this failure mode, visualized in the figure, on a new test image, the \textit{right} action has a high Q-value under the first ensemble member but low under the second, and the \textit{up} action has a low Q-value under the first ensemble member but high in the second. The stay action is neutral (neither helpful nor harmful) under both value functions in the ensemble, and is chosen by the LCB statistic since it appears to have the same net benefit (although this is not true in practice).

\section{Details for \textit{Procgen Mazes} Domain}
\label{appendix:procgen}

\subsection{Environment and Dataset Details}
The Procgen Maze task \citep{Cobbe2020LeveragingPG} requires an agent to control a mouse to reach a block of cheese (the goal) in some maze layout within 500 environment steps. Each time-step, the agent receives a render of the full environment as a $64 \times 64 \times 3$ image (which contains the mouse, the full maze layout, and the goal), and must take one of $15$ actions (the standardized Procgen interface). Since the whole maze is visible to the agent at each time-step, there is no partial observability and the environment is an MDP. We procedurally generate a dataset for the Procgen task from a set of training levels in the following way. For each maze, we enumerate the list of valid positions in the maze, then manually reset the agent to each position and take the \{Up, Down, Left, Right\} actions, logging the ensuing transitions. These per-maze datasets are concatenated together to form the full offline dataset (we create two offline datasets, one with $200$ training mazes, and another with $1000$ mazes).

\subsection{Implementation details}
\begin{table}[t!]
    \centering
    \small
    \caption{Hyperparameters used for training Q learning based agents in Procgen Mazes domain}
    \begin{tabular}{cc}
    \toprule
    Hyperparameter & Value \\ [0.5ex]
    \midrule
     $\gamma$ & 0.99 \\ [0.5ex]
     Reward shift & -1.0\\
     Distributional support & \textsc{linspace(-31, 9, 81)} \\
     Batch size & 256 \\ [0.5ex]
     Learning rate & 6.25e-5 \\ [0.5ex]
     Optimizer & Adam~\citep{kingma2014adam} \\ [0.5ex]
     Training steps & $10^6$ \\ [0.5ex]
     Number of ensembles & 2 \\ [0.5ex]
     $p(\vb)$ & SymmetricDirichlet(1.0) \\ [0.5ex]
    \bottomrule
    \end{tabular}
    \label{tbl:hyperparams_procgen}
\end{table}
For all of our comparisons, we train value functions using the C51 \citep{Bellemare2017ADP} algorithm for discrete max-Q learning, using prioritized experience replay \citep{Schaul2016PrioritizedER} to sample from the offline dataset. We parameterize the Q-function using the same Impala encoder as \citet{Cobbe2020LeveragingPG}, with a linear readout for the logits of the distributional value function. As in \textit{Locked Doors}, since APE-V additionally takes in a belief vector, we concatenate the belief vector to the output of the visual Impala encoder, and pass this on to the rest of the network. Aside from these environment-specific details, training is the same as in the Locked Doors domain -- exact hyperparameter details are provided in Table \ref{tbl:hyperparams_procgen}. 

\clearpage
\section{Details for \textit{D4RL} Benchmark}
\label{appendix:d4rl}
\subsection{Implementation Details}
For the D4RL benchmark, we instantiate \algoname~using SAC-n~\citep{an2021uncertainty} as our base value learning method since it has fewer issues of optimization than standard SAC on the D4RL tasks. SAC-n parameterizes a Q-function as the minimum of $n$ independent value functions (standard SAC is SAC-n with $n=2$)  \algoname~learns an ensemble of $K$ SAC-n agents, each with different values of $n$ and maintains a belief over them for test time adaptation. To better capture the uncertainty in the environment, we promote diversity amongst the SAC-n agents in the ensemble by choosing a different value of $n$ for each of them. We parameterize the actor using a set of $K$ independent networks as well to avoid potential challenges in optimizing different combinations of our $Q$-functions simultaneously: $\pi(\cdot | s, \vb) = \sum_{i=1}^K \vb_i \pi_i(\cdot | s)$ and use $p(\vb) = \text{SymmetricDirichlet}(.01)$. Since \algoname~contains an ensemble of $K$ SAC-n agents, each of which contains $n$ Q functions, learning \algoname~can become computationally intensive for large values of $n$. To remain computationally tractable, we implement weight sharing between Q functions of different SAC-n agents. Specifically, let $\{n_1, \dots n_K\}$ be the values of $n$ we wish to use for each of our ensemble members; we train $\max_i n_i$ ensembles $\{Q_1, \dots Q_{\max_i n_i}\}$, each with $K$ heads. Using this, we define the $i$-th SAC-n agent $Q_i^{\text{SAC-n}}$ as $Q_i^{\text{SAC-n}} = \min \{Q_1^i, \ldots, Q_{n_i}^i\}$. This construction ensures that all the networks within each SAC-n agent be independent, while avoiding creating an excessive number of ensembles.

We use hyperparameters from \citet{an2021uncertainty} (\href{https://github.com/snu-mllab/EDAC.git}{https://github.com/snu-mllab/EDAC.git}). The original implementation of SAC-$n$ uses separate values of $n$ for each domain: $10$ for half-cheetah, $20$ for walker, and $500$ for hopper. We train $N-2$ SAC-$n$ agents, with values of $n$ in $\{2, \dots, N\}$, where $N$ is the number of ensembles in the original implementation of SAC-$n$. With the exception of Half-Cheetah (which has the lowest number of ensembles), we implement weight sharing among SAC-n agents contained in \algoname~ to reduce ensemble training costs.

\begin{table*}
	\centering
	\normalsize
	\caption{Normalized average returns on D4RL suite, averaged over 4 random seeds.}
	\vspace{0.2em}
	\label{tab:gym_full}
	\begin{adjustbox}{max width=\linewidth}
		\begin{tabular}{l|cccccc|r}
			\toprule
			\multirow{2}{*}{\textbf{Task Name}} & \multirow{2}{*}{\textbf{BC}} & 
			\textbf{SAC} & 
			\textbf{REM} & 
			\textbf{CQL} & 
			\textbf{IQL} & 
			\textbf{SAC-$N$} & 
			\multirow{2}{*}{\textbf{\algoname}} \\
			& & \citep{haarnoja2018soft} & \citep{Agarwal2020AnOP} & \citep{kumar2020conservative} & \citep{kostrikov2021offline} & \citep{an2021uncertainty} & \\
			\midrule
			halfcheetah-random & 2.2$\pm$0.0 & 29.7$\pm$1.4 & -0.8$\pm$1.1 & \textbf{35.4} & 31.3$\pm$3.5 & 29.8$\pm$1.6 & 29.9$\pm$1.1 \\
			halfcheetah-medium & 43.2$\pm$0.6 & 55.2$\pm$27.8 & -0.8$\pm$1.3 & 44.4 & 47.4$\pm$0.2 & 67.5$\pm$1.2 & \textbf{69.1 $\pm$ 0.4} \\
			halfcheetah-medium-expert & 44.0$\pm$1.6 & 28.4$\pm$19.4 & 0.7$\pm$3.7 & 62.4 & 95.0$\pm$1.4 & \textbf{102.7$\pm$1.5} &\textbf{ 101.4 $\pm$ 1.4 }\\
			halfcheetah-medium-replay & 37.6$\pm$2.1 & 0.8$\pm$1.0 & 6.6$\pm$11.0 & 46.2 & 44.2$\pm$1.2 & \textbf{63.9$\pm$0.8} & \textbf{64.6 $\pm$ 0.9}  \\
			hopper-random & 3.7$\pm$0.6 & 9.9$\pm$1.5 & 3.4$\pm$2.2 & 10.8 & 5.3$\pm$0.6 & \textbf{31.3$\pm$0.0} & \textbf{31.3$\pm$0.2x} \\
			hopper-medium-expert & 53.9$\pm$4.7 & 0.7$\pm$0.0 & 0.8$\pm$0.0 & \textbf{111.0} & 96.9$\pm$15.1 & 110.1$\pm$0.3 & 105.72 $\pm$ 3.7 \\
			hopper-medium-replay & 16.6$\pm$4.8 & 7.4$\pm$0.5 & 27.5$\pm$15.2 & 48.6 & 94.7$\pm$8.6 & \textbf{101.8$\pm$0.5} & {98.5 $\pm$ 0.5} \\
			walker2d-random & 1.3$\pm$0.1 & 0.9$\pm$0.8 & 6.9$\pm$8.3 & 7.0 & 5.4$\pm$1.7 & \textbf{16.3$\pm$9.4} & \textbf{15.5$\pm$8.5} \\
			walker2d-medium & 70.9$\pm$11.0 & -0.3$\pm$0.2 & 0.2$\pm$0.7 & 74.5 & 78.3$\pm$8.7 & 87.9$\pm$0.2 & \textbf{90.3 $\pm$ 1.6} \\
			walker2d-medium-expert & 90.1$\pm$13.2 & 1.9$\pm$3.9 & -0.1$\pm$0.0 & 98.7 & 109.1$\pm$0.2 & \textbf{116.0$\pm$6.3} & 110.0 $\pm$ 1.5\\
			walker2d-medium-replay & 20.3$\pm$9.8 & -0.4$\pm$0.3 & 12.5$\pm$6.2 & 32.6 & 73.8$\pm$7.1 & 78.7$\pm$0.7 & \textbf{82.9 $\pm$ 0.4}\\
\bottomrule
		\end{tabular}
	\end{adjustbox}
\end{table*}
\vspace{-1em}

\end{document}